\documentclass[final]{article}
\usepackage{graphicx, graphics}
\usepackage{amssymb,amsmath,amsthm}
\usepackage{array, latexsym}
\usepackage{epstopdf}
\usepackage{subfigure}
\usepackage{fullpage}
\usepackage{setspace}
\usepackage{vmargin}
\usepackage{framed}
\usepackage{authblk}

\newtheorem{definition}{Definition}
\newtheorem{theorem}{Theorem}

\newcommand{\R}{\mathbb{R}}

\newcommand{\ben}{\begin{enumerate}}
\newcommand{\een}{\end{enumerate}}

\newcommand{\bo}{\textbf}
\newcommand{\mean}{\textbf{$\mu$}}
\newcommand{\x}{\textbf{x}}
\newcommand{\p}{\textbf{p}}
\newcommand{\aj}{\textbf{a}_j}
\newcommand{\ajhat}{\widehat{\textbf{a}_j}}
\newcommand{\I}{\textbf{I}}
\newcommand{\A}{\textbf{A}}
\newcommand{\e}{\textbf{e}}
\newcommand{\uu}{\textbf{u}}
\newcommand{\C}{\textbf{C}}
\newcommand{\Q}{\textbf{Q}}
\newcommand{\plane}{\textit{P}}

\begin{document}

\title{Principal Direction Gap Partitioning (PDGP)}
\author[1]{Ralph Abbey}
\author[2]{Jeremy Diepenbrock}
\author[3]{Amy Langville}
\author[1]{Carl Meyer}
\author[1]{Shaina Race}
\author[4]{Dexin Zhou}
\affil[1]{North Carolina State University}
\affil[2]{Washington University}
\affil[3]{College of Charleston}
\affil[4]{Bard College}
\maketitle


\section{Introduction}
Data clustering has various applications in a wide variety of fields ranging from social and biological sciences, to business, statistics, information retrieval, machine learning and data mining. Clustering refers to the process of grouping data based only on information found in the data which describes its characteristics and relationships. Although humans are generally very good at discovering patterns and classifying objects, clustering algorithms are able to discern similarities in data even when humans are not  \cite{chap8}. The main focus of our research has been document clustering, but we will demonstrate that our methods also work nicely on scientific data.

In this paper, we propose an adaptation of the clustering algorithm known as Principal Direction Divisive Partitioning (PDDP) developed by Daniel Boley in \cite{BoleyPDDP} which is based Principal Components Analysis (PCA). PCA involves the eigenvector decomposition of a data covariance matrix, or equivalently a singular value decomposition (SVD) of a data matrix after mean centering.  The name of our adaptation, Principal Direction Gap Partitioning (PDGP), borrows most of its name from PDDP as it follows many of the same steps that PDDP follows. The word ``gap'' replaces the word ``divisive'' in reference to how the algorithm splits data along natural gaps at each step. This concept will be further developed in the following sections, but it should be noted that PDGP is still a divisive algorithm in the same way that PDDP is.


\section{Mathematical Notation and Background}
 
 In order to fully understand how and why PDDP works, we will begin with a detailed description of the linear algebra and geometry which support the algorithm.\\

\begin{definition}{The Singular Value Decomposition (SVD)}
For each $\C\in\Re^{m\times n}$ of rank $r$, there are orthogonal matrices
\begin{equation}
 \textbf{U}_{m\times m} = [\uu_1|\uu_2|...|\uu_m]
 \hspace{10pt} \hbox{ and } \hspace{10pt}
 \textbf{V}_{n\times n} = [\textbf{v}_1|\textbf{v}_2|...|\textbf{v}_n]
\nonumber
\end{equation}
and a diagonal matrix $\textbf{D}_{r\times r} = \hbox{diag}(\sigma_1, \sigma_2, ..., \sigma_r)$ such that
\begin{equation}
 \C=\textbf{U}
 \left(\begin{array}{c c} \textbf{D} & \textbf{0} \\ \textbf{0} & \textbf{0} \\
 \end{array}\right)_{m\times n}
 \textbf{V}^T=\sum_{i=1}^r\sigma_i\uu_i\textbf{v}_i^T \hspace{10pt} \hbox{with} \hspace{10pt}
 \sigma_1\geq \sigma_2\geq ...\geq \sigma_r > 0.
\nonumber
\end{equation}
The $\sigma_i$'s are the nonzero singular values of \C, and the respective columns $\uu_j$ and the $\textbf{v}_j$ are the left-hand and right-hand singular vectors for \C. 
\end{definition}

\subsection{Directions and Lines of Principal Trend}
The principal trend in data can be considered in two ways. In principal component analysis (PCA) the direction of principal trend is considered the direction in which the variance (or spread) of the data is maximal \cite{PCA}. Another way to define the principal trend is by means of least squares, in which case the trend is along a line \L\, for which the total sum of squares of orthogonal deviations from \L\, is minimal among all lines in $\R^n$. The concepts of maximal spread and minimal deviations are equivalent in this context. For the sake of subsequent developments, we present the details of this fact below.

For a matrix $\bo{A}_{mxn} = [\bo{a}_1 \vert \bo{a}_2 \vert \dots \vert \bo{a}_n]$ of column data, we define the mean and variance, respectively, as follows:

\begin{eqnarray*}
 \bo{$\mu$} &=& \frac{1}{n} \sum_{i=1}^n \bo{a}_i = \frac{\bo{Ae}}{n} \\
 Var[A] &=& \frac{1}{n} \sum_{i=1}^n \|\bo{a}_i - \bo{$\mu$} \| =\frac{ \| \bo{A} - \bo{$\mu$}\bo{e}^T\|_F^2}{n}\\
            &=& trace \frac{(\bo{A} - \bo{$\mu$}\bo{e}^T)^T (\bo{A} - \bo{$\mu$}\bo{e}^T)}{n}\\
            &=& \frac{\|\bo{A}\|_F^2}{n} - \|\bo{$\mu$}\|_2^2
\end{eqnarray*}

Where \bo{e} is a vector of all ones and $\|*\|_F$ is the Frobenius matrix norm.  We will refer to a centered matrix, $\C=\A-\mean\e^T$, whose mean is zero and variance is$ \frac{\|\C\|_F^2}{n}$.
A trend line $\bo{L}(\bo{x},\bo{p}) = \{ \alpha\x + \p \vert \alpha \in \R\}$ for a data cloud in $\R^m$ is defined by a direction vector $\x \in \R^m$ with $\|\x\|_2 = 1$ and a point $\p \in \R^m$. See Figure \ref{fig:trendline} .

\begin{figure}[ht]
 \centering
 \includegraphics[scale=.75]{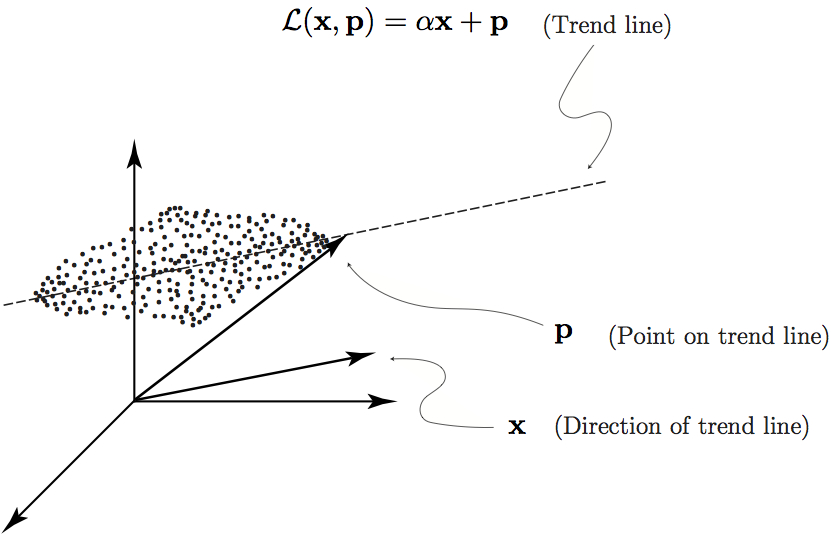}
 \caption{Trend Line}
 \label{fig:trendline}
\end{figure}

\subsubsection{Minimum Deviation Trend Line}
The minimum deviation trend line is the line \L\, for which the total sum of squares of orthogonal deviations between the data and \L\, is minimal among all lines in $\R^m$. To determine \L, let $\ajhat$ denote the orthogonal projection of $\aj$ onto a line \L(\x,\p). This orthogonal projection is given by
\begin{equation}
 \ajhat = \x\x^T(\aj - \p) + \p
\nonumber
\end{equation}
 and thus the difference between $\aj$ and the closest point on \L(\x,\p) is $\aj - \ajhat = (\I-\x\x^T)(\aj -\p)$. Consequently, the minimum deviation trend line is located by finding $\x,\p \in \R^m$ that solves the minimization problem 
\begin{equation}
 \displaystyle \min_{\x,\p,\|\x\|_2=1} f(\x,\p)
\nonumber
\end{equation}
where the objective function is 
\begin{equation}
 f(\x,\p) = \sum_{j=1}^n \| \aj - \ajhat \|_2^2 = \|(\I-\x\x^T)(\A -\p\e^T)\|_F^2
\nonumber
\end{equation}
The following theorem precisely characterizes the minimum deviation trend line.
 
\begin{theorem}[Minimum Deviation Trend Line] The minimum deviation trend line for the column data in \A\,  is given by
\begin{equation}
 \mbox{\L} = \{ \alpha \uu_1(C) + \mean \vert \alpha \in \R\}
\nonumber
\end{equation}
where $\uu_1(C)$ is the principal left-hand singular vector of the centered matrix
\begin{equation}
 \C=\A-\mean\e^T = \A(\I-\e\e^T/n)
\nonumber
\end{equation}
\end{theorem}

\begin{proof}
Apply straightfoward differerentiation to the function f(\x,\p), and begin by looking for points \p\,  that satisfy 
\begin{equation}
 0= \frac{\partial{f}}{\partial{\p}} = (\partial{f}/\partial{p_1}, \dots, \partial{f}/\partial{p_m})^T
\nonumber
\end{equation}
Letting $\Q = \I - \x\x^T$ and using $\Q^2 = \Q = \Q^T$ (since $\|\x\| = 1$) yields
\begin{eqnarray*}
 f(\x,\p) &=& trace([\A^T-\e\p^T]\Q[\A -\p\e^T])\\
            &=& trace(\A^T\Q\A) - 2trace(\A^T\Q\p\e^T) + trace (\e\p^T\Q\p\e^T) \\
            &=& trace(\A^T\Q\A) -2n\p^T\Q\mean + n\p^T\Q\p\\
\end{eqnarray*}
so that 
\begin{equation}
 \frac{\partial{f}}{\partial{\p}} = -2n\Q\mean + 2n\Q\p
\nonumber
\end{equation}
consequently, 
\begin{equation}
 \frac{\partial{f}}{\partial{\p}} = 0 \Longrightarrow \Q(\p-\mean) = 0
\nonumber
\end{equation}
and thus, $\p=\alpha\x+\mean$ where $\alpha = \x^T(\mean-\p)$. In other words, regardless of what \x\, turns out to be, a minimizing point $\p$  necessarily lies on the line \L(\x,\mean). Thus, to find the direction vector, \x, which minimizes f(\x,\mean), observe that
\begin{equation}
 f(\x,\mean)=\|(\I-\x\x^T)(\A-\mean\e^T)\|_F^2 = \|(\I-\x\x^T)\C\|_F^2 = \|\C\|_F^2 - \|C^T\x\|_2^2
\nonumber
\end{equation}
so the minimum of f(\x,\mean) is obtained precisely at points where $\displaystyle \max_{\|x\|_2=1} \|\C^T\x\|_2^2$ is obtained. It is well known that 
\begin{equation}
 \max_{\|\x\|_2=1} \|\C^T\x\|_2^2 = \|\C^T\|_2^2 = \sigma_1^2(\C)
\nonumber
\end{equation}
occurs at $\x = \uu_1(\C)$, and thus the minimum deviation (or total least squares) trend line is 
\begin{equation} 
 \mbox{\L}=\{\alpha\uu_1(\C) + \mean \vert \alpha \in \R\}
\nonumber
\end{equation}
\end{proof}


\subsubsection{Maximum Variance Trend Line}
Another natural way to gauge the principal trend of the data is to locate the line $\mbox{\L} \in \R^m$ along which the data is most spread-i.e., the line along which the variance is maximal. Since the orthogonal projection of $\aj$ onto any line \L(\x,\p) is $\ajhat = \x\x^T(\aj - \p) + \p$, The directional variance along \L(\x,\p) is 
\begin{equation}
 Var[\widehat{\A}] = \frac{\|\widehat{\A} - \mean_{\widehat{\A}}\e^T\|_F^2}{n} , \mbox{   where     } \widehat{\A} =
 (\I-\x\x^T)\p\e^T + \x\x^T\A
\nonumber
\end{equation}
Since $\mean_{\widehat{\A}} = \widehat{\A}\e/n = (\I - \x\x^T)\p + \x\x^T\mean$, it follows that 
\begin{equation}
 \widehat{\A} - \mean_{\widehat{\A}}\e^T = \x\x^T(\A - \mean\e^T) = \x\x^T\C
\nonumber
\end{equation}
and thus, 
\begin{equation}
 Var[\widehat{\A}]= \frac{\|\x\x^T\C\|_F^2}{n} = \frac{\|\C^T\x\|_2^2}{n}.
\nonumber
\end{equation}
So by the same reasoning above, the direction vector \x\, that maximizes the directional variance $Var[\widehat{\A}]$ is also
$\uu_1(\C)$ 

\begin{definition}[The Principal Trend Line]
The \textbf{principal trend line} for the column data in \A\, is defined to be 
\begin{equation}
 \mbox{\L}=\{\alpha\uu_1 + \mean \vert \alpha \in \R\}
\nonumber
\end{equation}
and it represents both the line of minimal total deviations as well as the line of maximal variance. \textbf{Note:} Unless otherwise stated, it is hereafter understood that $\uu_1 = \uu_1(\C)$ is the principal left-hand singular vector of the centered matrix $\C=\A-\mean\e^T = \A(\I-\e\e^T/n)$
\end{definition}


\subsubsection{Principal Partitions}
The first step in making principal partitions is to divide the data into two disjoint sets by slicing it with an affine hyperplane $\plane = \uu_1^\perp + \mean$ that is orthogonal to the principal trend line 
$\mbox{\L} =\alpha\uu_1 + \mean$. As depicted in Figure \ref{fig:affinehyperplane} it is natural to put the points that are on one side of \plane\, (say the points ``in front'' of \plane, as depicted in the figure) into one group and to put points on the other side (the points ``behind'' \plane) into another group.

\begin{figure}[ht]
 \centering
 \includegraphics[scale=.75]{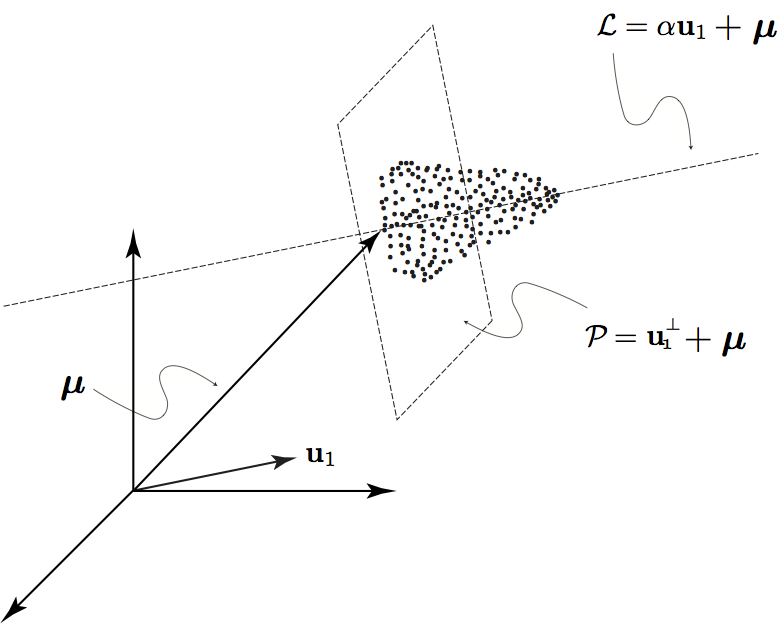}
 \caption{Data Cloud Partitioned by Affine Hyperplane}
 \label{fig:affinehyperplane}
\end{figure}

The distinction between ``front'' and ``back'' is simply made by determining whether the projection $\widehat{\aj}  = \uu_1\uu_1^T(\aj - \mean)+\mean$ of a data point $\aj$ onto the principal trend line $\mbox{\L}=\alpha\uu_1+\mean$ lies to one side of \mean\, or the other. Since $\widehat{\aj} - \mean = \alpha_j\uu_1$ for some $\alpha_j$, the sign of $\alpha_j$ determines the side of \plane\, that $\widehat{\aj}$ and $\aj$ are on. Since $\alpha_j = \uu_1^T(\aj -\mean)$ and since $\aj-\mean = \textbf{c}_j$ is the $j^{th}$ column of the centered matrix \C, it follows that
\begin{equation}
 [\alpha_1, \alpha_2, \dots, \alpha_n] = [\uu_1^T \textbf{c}_1, \uu_1^T\textbf{c}_2, \dots, \uu_1^T\textbf{c}_n] = \uu_1^T\C =
 \sigma_1\textbf{v}_1^T
\nonumber
\end{equation}
where $\textbf{v}_1$ is the principal right-hand singular vector of \C\, that is associated with the largest singular value, $\sigma_1$. The fortunate aspect of this observation is that once the SVD of \C\, has been computed, the vector $\sigma_1\textbf{v}_1^T$ is immediately available. Furthermore, since only the signs of the components in $\uu_1^T\C$ are needed to determine to which side of \plane\, the respective columns in \A\, lie, and since $\sigma_1 > 0$, it is evident that the principal partition is determined simply by inspecting the signs of the entries in $\textbf{v}_1.$
 
\begin{definition}[The Principal Partition]
The \textbf{principal partition} of the column data in \A\, is determined by the signs of the entries in the principal right-hand singular vector, $\textbf{v}_1$ of the centered matrix \C. Columns in \A\, corresponding to positive signs in $\textbf{v}_1$ are placed in one cluster while columns corresponding to negative signs are placed in another cluster. A column associated with a zero entry in $\textbf{v}_1$ may be arbitrarily assigned to either cluster.
\end{definition}

\section{Principal Direction Divisive Partitioning}
Once the principal partition of the data has been made, there are several options for making further partitions. One such option is the \textit{principal direction divisive partitioning} (PDDP) scheme proposed by Daniel Boley \cite{BoleyPDDP}. This algorithm suggests we make the principal partition and then examine both clusters to determine which has the maximal variance, or scatter. This cluster of maximal variance is then repartitioned across its own principal trend line, separating the data into a total of three disjoint sets, and the process continues by repartitioning the cluster of maximal variance each time, producing any desired number of disjoint (hard) clusters. At each step of PDDP the projected data is split by a principal partition.
\section{Principal Direction Gap Partitioning}
Principal Direction Gap Partitioning (PDGP) is our adaptation of PDDP which takes into account natural gaps which identify clusters in the data. We will motivate our algorithm with some the discussion of some geometrical scenarios in which PDDP breaks down.

\subsection{Motivation}
The technique of clustering the column data in \A\, by means of principal partitions is appealing because it is easily implemented by simply inspecting the signs of the principal right-hand singular vector of \C. However, superior results can often be obtained if we are willing to compromise this simplicity slightly to look for natural gaps in the data. For example, suppose that the data naturally clusters into three distinct gaps as shown in Figure \ref{fig:threeclusters}.

\begin{figure}[h]
 \centering
 \includegraphics[scale=.75]{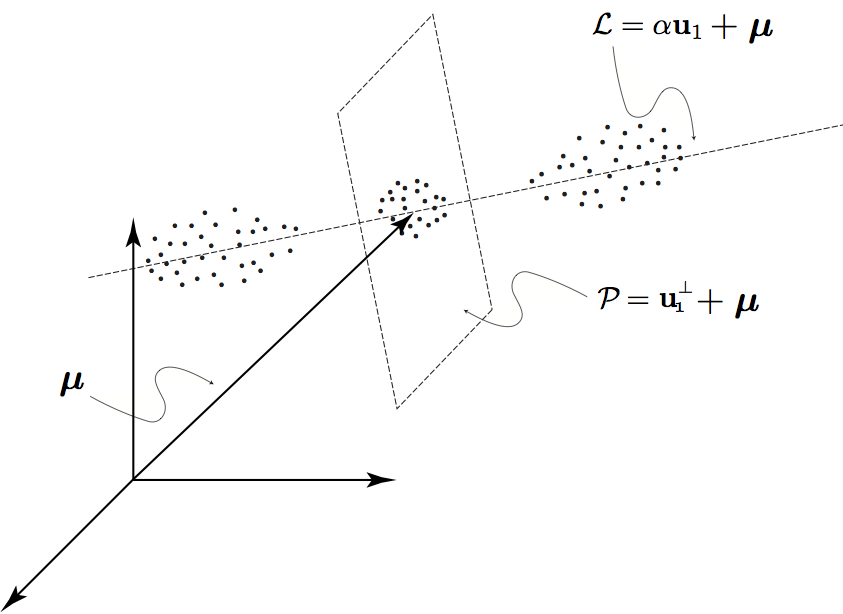}
 \caption{Three Data Clouds along Principal Direction}
 \label{fig:threeclusters}
\end{figure}

If this data is partitioned by \plane\, using the signs of $\textbf{v}_1$, then the middle cluster is unnaturally sliced into two pieces. It seems more reasonable to shift \plane\, and partition the data with an affine hyperplane $(\alpha\uu_1+\mean) + \uu_1^\perp$ as shown in Figure \ref{fig:threeclustersplane} where $\alpha$ is chosen to put the shifted hyperplane into the largest gap in the data.

\begin{figure}[ht]
 \centering
 \includegraphics[scale=.75]{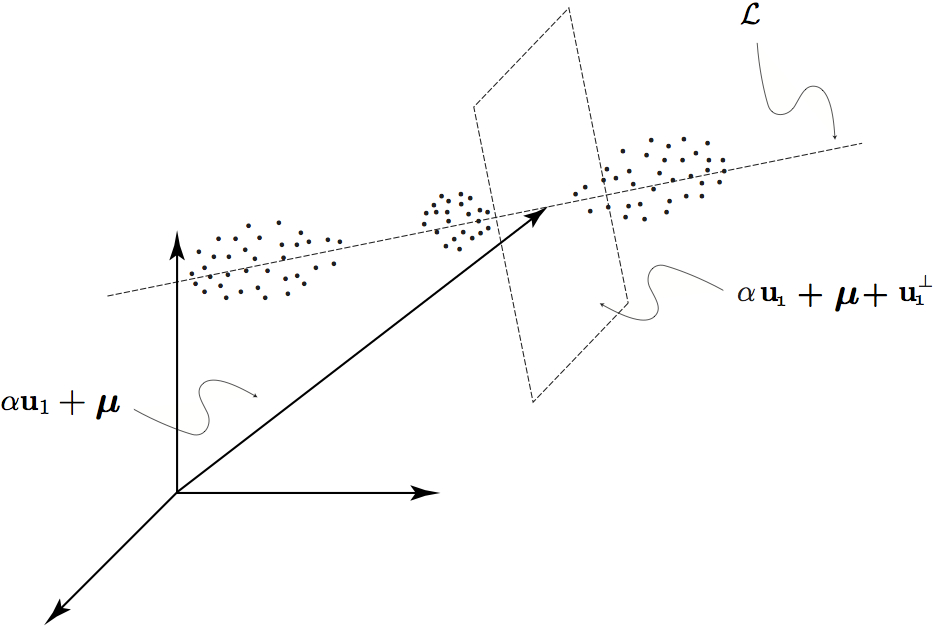}
 \caption{Partition by a Shifted Affine Hyperplane}
 \label{fig:threeclustersplane}
\end{figure}
  
Gaps between clusters are easily detected by projecting the columns of \A\, onto the principal trend line and measuring the gaps between adjacent points. 


As admitted in \cite{BoleyPDDP}, the choice of splitting the projected data at zero is somewhat arbitrary because it is based on the assumption that the mean of the data will naturally fall in between two well separated clusters. It is easy to see when this assumption might fail, for example in the case of unbalanced cluster sizes. Figures \ref{fig:dotplots1} and \ref{fig:dotplots2} show two real world examples in which this assumption fails. In these two graphs the entries in the principal right-hand singular vector, $\textbf{v}_1$,  are plotted in increasing order. The black line depicts the split that PDDP will make. The PDGP algorithm splits at the gap that naturally clusters the data.
\begin{figure}[ht]
 \centering
 \includegraphics[scale=.75]{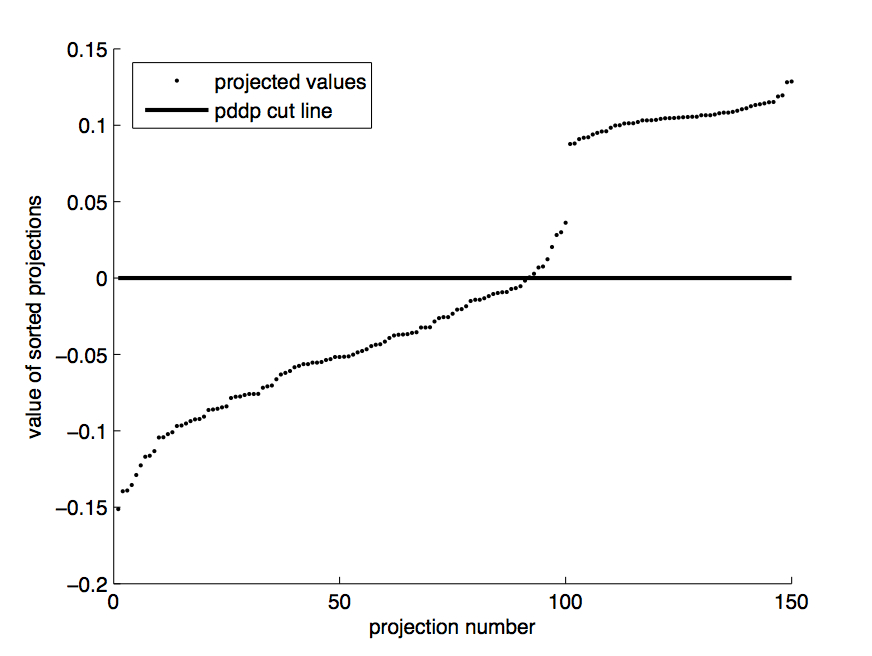}
 \caption{Principal Partition of projected data points }
 \label{fig:dotplots1}
\end{figure}

\begin{figure}[ht] 
 \centering
 \includegraphics[scale=.75]{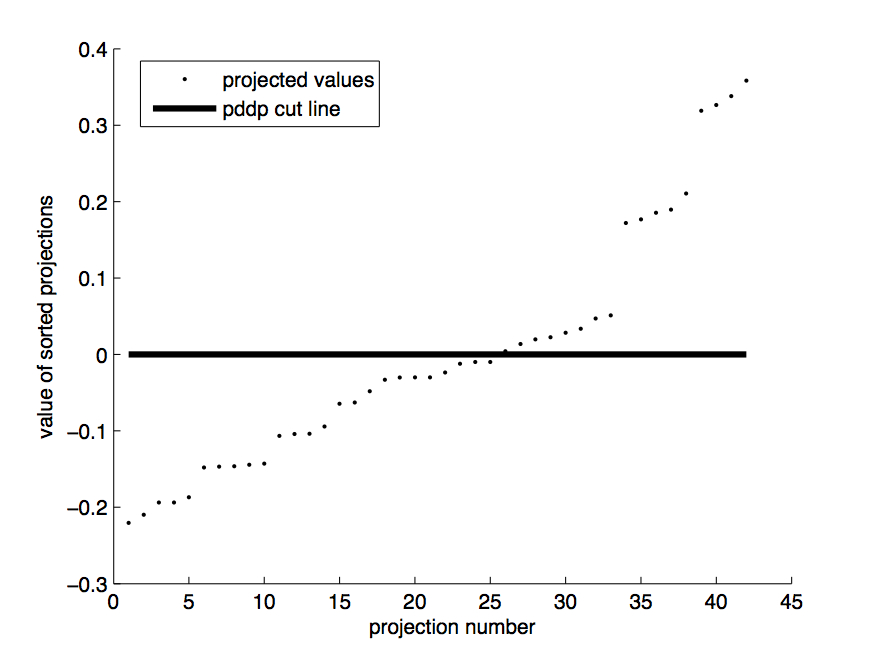}
 \caption{Principal Partition of projected data points}
 \label{fig:dotplots2}
\end{figure}

Sometimes the division made by PDDP and PDGP coincide. This indicates a situation when the assumption that two clusters are separated by the mean holds true. Figure \ref{fig:dotplots3}

\begin{figure}[ht]
 \centering
 \includegraphics[scale=.75]{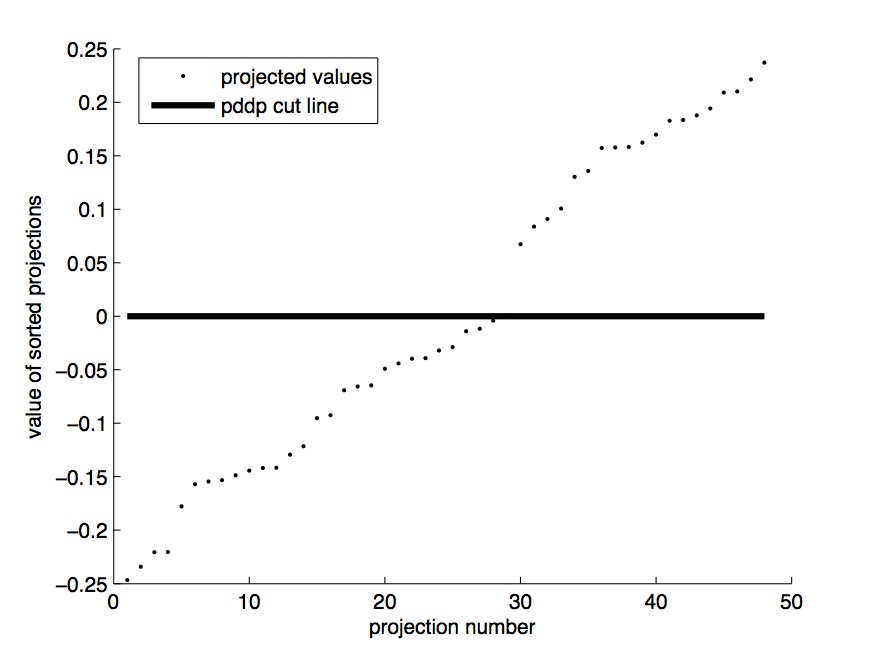}
 \caption{PDDP and PDGP divisions coincide}
 \label{fig:dotplots3}
\end{figure}

\subsection{Description of the Algorithm}
The PDGP algorithm is identical to the PDDP algorithm aside from where the data is split at each step. After the data is projected onto the principal trend line, PDDP splits the data at 0 while PDGP splits the data at the largest gap between the points. To further clarify this, we propose the following definition.
\begin{definition}[Gap Partition]
Sort the components of the first right-hand singular vector, $\textbf{v}=\textbf{v}_1$ in ascending order and label the sorted vector \textbf{s}. Let $p$ be the permutation required for this sort, i.e. $s_1 =v_{p_1} \leq v_{p_2} =s_2 \leq \dots \leq v_{p_n} = s_n$. If the maximum value of $\textbf{s}$ occurs at $\textbf{s}_k$ then the \textbf{gap partition} of \textbf{v}, which provides the indices of the column vectors that should be placed in the respective cluster, is:
$$\Pi = \begin{cases} \pi_1 = [p_1 , \dots, p_k] \\
				  \pi_2 = [p_{k+1}, \dots, p_n]
				  \end{cases}$$
\end{definition}

\subsubsection{Fringe Effect and Fringe Tolerance}
One obstacle in the implementation of this algorithm is something we call the \textit{fringe effect}. This is where the gap in a vector $\textbf{v}_1$ occurs very close to the ends of \textbf{s}. These ``fringe gaps", if taken into account, would separate the data into severely unbalanced clusters, one containing almost all of the data and the second containing a mere few. Because the fringe points are often depict outliers or noise, this issue must be addressed. \\
For an example, see Figure \ref{fig:fringe}. Notice on either end of the the outlying data points that create large gaps. The human eye is likely to find 3 or 5 clusters in this image, depending on whether you decide the first and last points belong to their own cluster or not. Since one of our goals is to find relatively balanced clusters, the ideal split appears to be between 18 and 19 or between 25 and 26. The line shown is how PDDP would split the data into clusters.

\begin{figure}[ht]
 \centering
 \includegraphics[scale=.75]{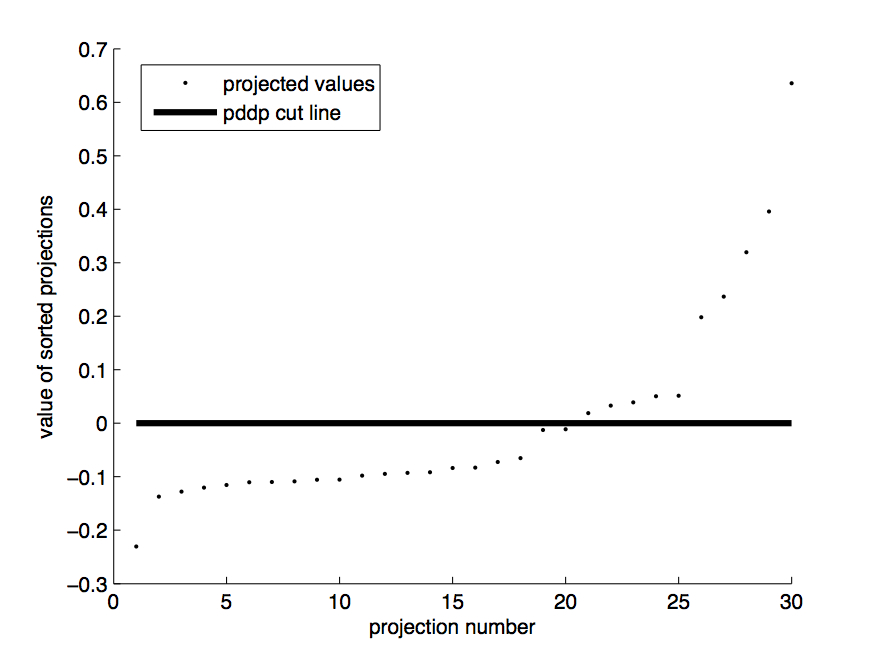}
 \caption{"Fringe values"}
 \label{fig:fringe}
\end{figure}

To counteract this phenomenon we created a ``fringe tolerance", $\tau$, to control the balance of cluster sizes.  We ignore a percentage of the projected data points at each end of the graph.  For our experiments we have ignored a total of 20 percent ($\tau=.2$), or 10 percent from each end. In choosing this particular value, we are insisting that the algorithm not separate the number of data points in a cluster into any ratio larger than 9:1. The fringe tolerance can be changed as the data set changes. Intuitively for smaller data sets the fringe tolerance should be higher, and for larger data sets it should be smaller, especially for a large number of clusters, as a lower percent still encompasses many data points.\\
The PDGP algorithm is identical to the PDDP algorithm aside from where the data is split at each step. After the data is projected onto the principal trend line, PDDP splits the data at 0 while PDGP splits the data at the largest gap between the points. \\
\begin{framed}
\textbf{PDGP Algorithm} 
\begin{itemize}

 \item[1.] Input: Data matrix A, desired number of clusters
 \item[2.] Determine the cluster of maximal variance, use these data vectors to form a matrix \textbf{M}. To begin with, this will be your entire data collection \textbf{A}.
\item[3.] Calculate the SVD of the centered matrix $\C=\textbf{M} - \mean \e^T $
\item[4.] Calculate the gap partition of the first right-hand singular vector, ignoring the proportion of data $\tau/2$ at the beginning and end of the sorted vector, \textbf{s}.
\item[5.] Repeat steps 2-4 until $k$ clusters are created.
\end{itemize}
\end{framed}

\section{Document Clustering}
The main focus of our research has been in the realm of document clustering. Data taken from a group of text documents is traditionally stored in an m$\times$n \textit{term-by-document matrix} where the m rows correspond to the various terms extracted from the documents and the n columns correspond to individual documents. The terms extracted from the document list are filtered through a``stoplist" of common words to remove terms like ``is,'' ``the,'' and ``however.''  The $A_{ij}$ entry of this matrix is the number of times term i occurs in document j.
\subsubsection{Term Weighting}
In the field of text-mining, the raw term-frequences in the term-document matrix, $\A$, are generally weighted in an effort to downplay the effects of commonly used words and bolster the effect of rare but semantically important words. In another approach the columns of $\A$ can be normalized so that lengthy documents do not overshadow their terse counterparts. In this paper we use TFIDF (term frequency - inverse document frequency) weighting or normalization scaling to pre-process our text data prior to clustering. For comparison to the PDDP algorithm we use the variant of TFIDF given in \cite{boleyTFIDF} as used in \cite{BoleyPDDP}. This variant of TFIDF is as follows:

\begin{definition}{Term Frequency - Inverse Document Frequency (TFIDF) weighting}
The Term Frequency - Inverse Document Frequency (TFIDF) weighting for a matrix A is \cite{boleyTFIDF},
\begin{equation}
  \widehat{\textbf{a}}_{ij}=\frac{1}{2}(1+\frac{\textbf{a}_{ij}}{\max_k(\textbf{a}_{kj})})*(\log_2(\frac{n}{\mbox{number of documents containing term i}})
\nonumber
\end{equation}
\end{definition}

For normalization scaling, each document is normalized to have unit Euclidean length:
\begin{equation}
 \widehat{\textbf{a}_{ij}}=\frac{a_{ij}}{\sqrt{\sum_k{a^2_{kj}}}}
\nonumber
\end{equation}

This can alternatively be thought of normalizing each document vector
\begin{equation}
 \widehat{d}_j=\frac{\mbox{d}_j}{\|d_j\|_2}
\nonumber
\end{equation}

It should be noted however, that in later sections of this paper we discuss the use of scientific data as well as the use of textual data. The rationale for term weighting in scientific data no longer applies because there are not documents of different length to contend with. Although normalization is commonly used in scientific data, it is unnecessary when the values of a variable are physically constrained to stay in a reasonable range

\subsection{Cluster Evaluation}
Cluster evaluation or  validation is an important aspect of any cluster related research. Since most existing clustering algorithms will determine clusters in data whether or not they exist naturally, it is important to have some way to evaluate the accuracy of clustering results. Cluster evaluation measures are typically broken into two catagories, internal (or unsupervised), and external (or supervised). Internal measures use no outside information, such as class or catagory labels, to determine the validity of the clustering. Internal measures are typically measures of cluster cohesion and separation. Cluster cohesion gives us an idea of how dense an individual cluster is while cluster separation tells us how distinct or separated the clusters are from each other.  The most commonly used internal measure is the silhouette coefficient, which combines both cohesion and separation. Since internal measures do not tell us explicitly about the accuracy of our clustering results, we do not use them in this paper. \cite{chap8}.

External measures use information not included in the dataset (such as class labels) to determine how well the algorithm clustered the data into its pre-determined clusters. External measures are not useful in practice because there is no need to cluster data which is already catagorically assigned, but they give us a more accurate metric for comparing different clustering algorithms. The most common external measure for cluster evaluation is entropy and is described in detail at the end of the paper. A smaller entropy value indicates a higher quality clustering. It is important to note that we have used normalized entropy so that the values fall between 0 and 1. For this reason, our entropy values for PDDP run on the same data sets as Boley's original experiments in \cite{BoleyPDDP} will differ by the multiplicative constant $\log_2(k)$ where k is the actual number of clusters.

\section{Description of Data Sets and Experimental Results}
Experiments comparing PDDP and PDGP were performed on a series of data sets. We chose not only document data sets, but also scientific data to compare the clustering algorithms.

\subsection{J Document Sets}
This document set was used in the original paper on PDDP\cite{BoleyPDDP}, and consists of 185 documents taken from the world wide web. A stop list of common words was applied, and also a stemmer to handle verb tenses, plurals, etc. By counting the rest of the words the resulting matrix was called J1. Further modifcations were made resulting in J2-J11 as seen in table 3 in \cite{BoleyPDDP}.

In running the two algorithms a table similar to table 5 in \cite{BoleyPDDP} was created. It is important to note that \textbf{normalized entropy} (see section on entropy) was used so that the values would range between 0 and 1. To produce tables similar to those in \cite{BoleyPDDP} simply scale these tables by a factor of $log_2(10)$ because it was predetermined that the data had 10 clusters. Smaller entropy values indicate a better clustering.

\begin{table}[ht]
\caption{Normalized entropy values obtained by PDDP and PDGP on the J document sets using norm scaling}
\begin{center} \footnotesize
\begin{tabular}[ht]{|c|ccc|ccc@{}c|}
\hline
data & \multicolumn{3}{c}{PDDP} & \multicolumn{3}{c}{PDGP} & \\
\cline{2-8}
sets & \multicolumn{3}{c}{norm scaling} & \multicolumn{3}{c}{norm scaling} & \\
\hline
clusters & 8 & 16 & 32 & 8 & 16 & 32 & \\
\hline
J1 & 0.372 & 0.208 & 0.154  & 0.388 & 0.191 & 0.147  & \\
J6 & 0.399 & 0.250 & 0.169  & 0.398 & 0.232 & 0.157  & \\
J4 & 0.459 & 0.331 & 0.213  & 0.508 & 0.319 & 0.215  & \\
J3 & 0.399 & 0.256 & 0.183  & 0.388 & 0.232 & 0.154  & \\
J7 & 0.408 & 0.270 & 0.182  & 0.393 & 0.220 & 0.158  & \\
J8 & 0.442 & 0.288 & 0.207  & 0.400 & 0.230 & 0.173  & \\
J2 & 0.510 & 0.337 & 0.229  & 0.449 & 0.302 & 0.214  & \\
J9 & 0.496 & 0.322 & 0.228  & 0.507 & 0.327 & 0.213  & \\
J10 & 0.507 & 0.351 & 0.257  & 0.523 & 0.381 & 0.225  & \\
J5 & 0.395 & 0.221 & 0.155  & 0.375 & 0.200 & 0.155  & \\
J11 & 0.443 & 0.315 & 0.202  & 0.510 & 0.343 & 0.217  & \\
\hline
\end{tabular}
\end{center}
\label{normscale}
\end{table}

\begin{table}[ht]
\caption{Normalized entropy values obtained by PDDP and PDGP on the J document sets using TFIDF term weighting}

\begin{center} \footnotesize
\begin{tabular}[ht]{|c|ccc|ccc@{}c|}
\hline
data & \multicolumn{3}{c}{PDDP} & \multicolumn{3}{c}{PDGP} & \\
\cline{2-8}
sets & \multicolumn{3}{c}{TFIDF} & \multicolumn{3}{c}{TFIDF} & \\
\hline
clusters & 8 & 16 & 32 & 8 & 16 & 32 & \\
\hline
J1 & 0.440 & 0.318 & 0.214 & 0.551 & 0.435 & 0.344 & \\
J6 & 0.353 & 0.232 & 0.194 & 0.546 & 0.401 & 0.272 & \\
J4 & 0.496 & 0.356 & 0.281 & 0.481 & 0.379 & 0.292 & \\
J3 & 0.472 & 0.335 & 0.278 & 0.487 & 0.350 & 0.263 & \\
J7 & 0.384 & 0.275 & 0.217 & 0.536 & 0.401 & 0.274 & \\
J8 & 0.398 & 0.272 & 0.230 & 0.491 & 0.361 & 0.300 & \\
J2 & 0.469 & 0.343 & 0.242 & 0.512 & 0.416 & 0.312 & \\
J9 & 0.428 & 0.308 & 0.228 & 0.472 & 0.311 & 0.221 & \\
J10 & 0.571 & 0.374 & 0.296 & 0.485 & 0.349 & 0.275 & \\
J5 & 0.322 & 0.184 & 0.138 & 0.439 & 0.263 & 0.170 & \\
J11 & 0.506 & 0.358 & 0.277 & 0.467 & 0.299 & 0.228 & \\
\hline
\end{tabular}
\end{center}
\label{tfidfscale}
\end{table}

It is apparent from the results in table \ref{normscale} that PDGP is competitive with PDDP when using the norm scaling, and frequently provides a clustering with lower entropy. When TFIDF weighting (Table \ref{tfidfscale} is used, PDDP performs slightly better than PDGP. However, it is experimentally evident that norm scaling provides better clustering results overall when compared to TFIDF weighting, and thus norm scaling should probably be used in favor of the TFIDF weighting for either of these two clustering algorithms.


\subsection{Abalone Data Set}
This data set was obtained from \cite{DataSets} and contains measurements of 4177 different abalone. There were 8 characteristics measured: sex (male, female, infant), length, diameter, height, whole weight, shucked weight, viscera weight, and shell weight. The sex variable was assigned to be 0 for a male, 1 for a female, and 2 for an infant. These measurements were paired with ages, of which 28 different ages were determined. We omitted the age variable from the dataset and aimed to cluster the organisms based upon this variable. However, there were several age groups containing only a few abalone (specifically the older age groups), and thus the sizes of the clusters are expected to be unbalanced. We used both PDDP and PDGP to cluster the data with various numbers of clusters. No scaling or normalization was applied to the data set because the values of each variable are expected to fall within a natural range.

\begin{table}[ht]
\caption{Normalized entropy values obtained PDDP and PDGP on the Abalone scientific data}

\begin{center} \footnotesize
\begin{tabular}[ht]{|c|ccccccccc|}
\hline
clusters sought & 20 & 21 & 22 & 23 & 24 & 25 & 26 & 27 & 28 \\
\cline{2-10}
& \multicolumn{8}{c}{Entropies} & \\
\hline
PDDP & 0.624 & 0.622 & 0.622 & 0.622 & 0.622 & 0.620 & 0.620 & 0.620 & 0.618 \\
PDGP & 0.622 & 0.620 & 0.618 & 0.618 & 0.616 & 0.616 & 0.616 & 0.614 & 0.614 \\
\hline
\end{tabular}
\end{center}
\label{abalone}
\end{table}






\subsection{Iris Data Set}
This data set was obtained from \cite{DataSets} and contains information on 150 flowers. Each flower was measured with four characteristics: sepal length, sepal width, petal length, and petal width. Of these flowers there are 3 different species, so the overall data was stored as a $4\times 150$ matrix. Again, no scaling or normalization was used. PDDP and PDGP were set to run to find 3 clusters.

\begin{table}[ht]
\caption{Normalized entropy values obtained by PDDP and PDGP on the Iris data}
\begin{center} \footnotesize
\begin{tabular}[ht]{|c|ccc|ccc@{}c|}
\hline
& \multicolumn{3}{c}{PDDP} & \multicolumn{3}{c}{PDGP} & \\
\hline
Iris Species & Cluster 1 & Cluster 2 & Cluster 3 & Cluster 1 & Cluster 2 & Cluster 3 & \\
\hline
\# of Setosa &50 & 0& 0& 50& 0& 0 & \\
\# of Versicolour & 9& 38& 3& 0& 50& 0 & \\
\# of Virginica & 0& 14& 36& 0& 34& 16 & \\
\hline
Total Entropy & & 0.404 & & & 0.347 & & \\
\hline
\end{tabular}
\end{center}
\label{iris}
\end{table}

\subsection{Reuters-10 Document data}
This collection of documents, downloaded from \cite{DataSets}, is a subset of the Reuters collection consisting of 20 documents pulled from each of 10 keyword searches for a total of 200 documents. The files were read out by 3 Indian speakers and an Automatic Speech Recognition (ASR) system was used to generate the transcripts. This dataset was collected to study the effect of speech recognition noise on text mining algorithms. Normalization scaling was used. In this noisy data set, PDDP provides a slightly lower entropy than PDGP.

\begin{table}[ht]
\caption{Normalized entropy values obtained by PDDP and PDGP on Reuters-10}

\begin{center} \footnotesize
\begin{tabular}[ht]{|c|c|}
\hline
& Entropy  \\
\hline
PDDP & .6021 \\
PDGP &  .6385 \\
\hline
\end{tabular}
\end{center}
\label{reuters}
\end{table}


\subsection{Wisconsin Breast Cancer Data Set (Original)}

This data set, also downloaded from \cite{DataSets} consisted of 699 observations of individuals with abnormal breast tissue growth obtained from the University of Wisconsin Hospitals, Madison by Dr. William H. Wolberg \cite{BC}. Each observation consists of 9 measurements such as clump thickness, uniformity of cell size and shape. A variable indicating whether a growth was benign or malignant was included in the data so we removed it and clustered the observations into two groups hoping to predict this variable through clustering. There were 16 missing values in the data which were set to 0. No normalization or scaling was applied. Both PDDP and PDGP performed well on this task, though PDGP was slightly superior. 

\begin{table}[ht]
\caption{Normalized entropy values obtained by PDDP and PDGP on Wisconsin Breast Cancer Data}

\begin{center} \footnotesize
\begin{tabular}[ht]{|c|c|}
\hline
& Entropy \\
\hline
PDDP & .0052 \\
PDGP &  .0029 \\
\hline
\end{tabular}
\end{center}
\label{breast}
\end{table}

\section{Conclusion}
PDDP/PDGP are both SVD based clustering algorithms which seek to use the principal trends in a given data set to separate related observations/documents into clusters. Where PDDP arbitrarily makes this split along the principal direction at the mean, PDGP looks for natural gaps in the data.  We sought to elucidate the geometrical interpretation of the singular vectors, and argue that although PDDP and PDGP often perform comparably,  gap partitioning makes more sense intuitively.  We have explored many variants of these clustering algorithms in our research, and have suggested some simple implementations for future research.

One of the issues at large with the PDGP algorithm is the fringe effect. The tolerance $\tau$ effectively controls the balance of the cluster sizes, but it arbitrarily causes the splitting algorithm to ignore a certain percentage of the data projections. There may be other applications that will allow for the inclusion of this information, for instance outlier identification. Especially in cases where documents are extracted from the world wide web it is likely that some noisy documents which have no connection to the other documents will be extracted. However, just because a projected point looks like an outlier along the principal directions doesn't mean that it is truly an outlier in the context of the whole data set. Looking along secondary directions may provide more information to this effect.

\section{Acknowledgements} We would like to thank the National Science Foundation (NSF) for funding our REU program and making our work possible. We are also grateful for the UC Irvine Machine Learning Repository of data sets. We downloaded several of the above mentioned data sets from their website, including the Wisconsin Breast Cancer Database which is described in detail at http://archive.ics.uci.edu/ml/machine-learning-databases/breast-cancer-wisconsin/breast-cancer-wisconsin.names. Thanks to Dimitrios Zeimpekis and Efstratios Gallopoulos, the creators of the MATLAB Text to Matrix Generator (TMG), which was used to parse many of the documents sets used herein.

\bibliography{PDGP_SIAM.bib}
\bibliographystyle{siam}	

\end{document}